\documentclass{article}
\pdfoutput=1

    \usepackage[preprint]{neurips_2020}

\usepackage[utf8]{inputenc} 
\usepackage[T1]{fontenc}    
\usepackage{hyperref}       
\usepackage{url}            
\usepackage{booktabs}       
\usepackage{amsfonts}       
\usepackage{nicefrac}       
\usepackage{microtype}      

\usepackage{amsmath}   
\usepackage{amssymb}
\usepackage{wasysym}
\usepackage{multirow}
\usepackage[dvipsnames]{xcolor}
\usepackage{amsthm}
\usepackage{bbm}
\usepackage{tikz}
\usepackage{enumerate}
\usepackage{wrapfig}
\usetikzlibrary{calc}
\usetikzlibrary{arrows.meta}
\usepackage{footmisc}

\newcommand{\yuyan}[1]{[\textcolor{cyan}{Yuyan: {#1}}]}

\usepackage{verbatim}
 
\usepackage{amsmath,amsfonts,bm,mathtools}

\def\eqref#1{equation~\ref{#1}}

\def\1{\bm{1}}

\DeclareMathAlphabet{\mathsfit}{\encodingdefault}{\sfdefault}{m}{sl}
\SetMathAlphabet{\mathsfit}{bold}{\encodingdefault}{\sfdefault}{bx}{n}

\newtheorem{pro}{Proposition}

\theoremstyle{remark}

\def\L{\mathcal{L}}
\def\Lhat{\hat{\mathcal{L}}}

\def\Y{\mathcal{Y}}

\newcommand{\ltwonorm}[1]{\lVert#1\rVert_2}
\graphicspath{{figures/}}
\usepackage{caption}
\usepackage{subcaption}

\title{Small Towers Make Big Differences}

\author{%
  Yuyan Wang, Zhe Zhao, Bo Dai, Christopher Fifty, Dong Lin, Lichan Hong, Ed H. Chi \\
  Google Research, Brain Team\\
  \texttt{\{yuyanw, zhezhao, bodai, cfifty, dongl, lichan, edchi\}@google.com}
}

\begin{document}

\maketitle

\begin{abstract}


Multi-task learning aims at solving multiple machine learning tasks at the same time. A good solution to a multi-task learning problem should be generalizable in addition to being Pareto optimal. In this paper, we provide some insights on understanding the trade-off between Pareto efficiency and generalization as a result of parameterization in multi-task deep learning models. As a multi-objective optimization problem, enough parameterization is needed for handling task conflicts in a constrained solution space; however, from a multi-task generalization perspective, over-parameterization undermines the benefit of learning a shared representation which helps harder tasks or tasks with limited training examples. A delicate balance between multi-task generalization and multi-objective optimization is therefore needed for finding a better trade-off between efficiency and generalization. To this end, we propose a method of \textit{under-parameterized self-auxiliaries} for multi-task models to achieve the best of both worlds. It is task-agnostic and works with other multi-task learning algorithms. Empirical results show that \textit{small towers} of under-parameterized self-auxiliaries can make \textit{big differences} in improving Pareto efficiency in various multi-task applications.



\end{abstract}

\section{Introduction}
\label{introduction}
In many machine learning applications, there are more than one task that is of interest. For example, an object detection algorithm may involve predicting both type and position of an object \citep{girshick2015fast}; a content recommendation system may care about both short-term conversion rate and long-term retention probabilities of the users \citep{zhao2019recommending}. These use cases require the prediction of multiple targets based on shared inputs, the solution of which is multi-task learning \citep{caruana1997multitask}. Over the past years, multi-task deep learning has gained popularity through its success in a wide range of applications, including natural language processing \citep{collobert2008unified}, computer vision \citep{girshick2015fast, ren2015faster}, and online recommendation systems \citep{bansal2016ask, ma2018modeling, ma2019snr}. 

A common modeling approach for multi-task learning is to design a parameterized model class that shares a subset of parameters across different tasks. The benefits of such a shared architecture are multi-fold. First, it exploits task relatedness with inductive bias learning \citep{caruana1997multitask, baxter2000model}. Assuming that tasks share a common hypothesis class, learning a shared representation across tasks is beneficial especially for harder tasks or tasks with limited training examples. Secondly, by forcing tasks to share model capacity, it introduces a regularization effect and improves generalization. Last but not least, it offers a much more compact and efficient model architecture compared with training each task separately, which unlocks potential in training and serving multiple tasks for large-scale systems. 

However, such a shared model architecture also introduces new challenges for learning. With different and potentially conflicting tasks as multiple objectives, it is unlikely in practice that all objectives achieve optimality at the same time \citep{sener2018multi, ma2018modeling}. In other words, multi-task learning naturally comes with trade-off between the performances on different tasks. Therefore, the solution to a multi-task learning problem is usually not a single solution, but rather, a set of solutions representing different trade-off decisions. Those solutions are not dominated by any others and are said to be Pareto optimal. Multi-task model performance can be evaluated by a set of Pareto optimal solutions, which form the Pareto frontier. Recent multi-task learning research has been focusing on better optimizing conflicting objectives to improve Pareto efficiency. Existing efforts include designing more flexible model architecture \citep{ruder2017overview} to handle task trade-offs and reduce task conflicts, as well as developing efficient optimization algorithms \citep{sener2018multi, chen2017gradnorm}.

In this paper, we empirically show that by increasing parameterization, multiple training objectives can have less conflicts. However, too much over-parameterization could yield solutions that are worse than those from smaller models. We show that on one hand, from the perspective of multi-objective optimization \citep{sawaragi1985theory}, enough parameterization is needed for properly handling task conflicts in a constrained solution space. On the other hand, over-parameterization diminishes the benefit of inductive transfer from multi-task deep learning and can lead to poor per-task performances.While bigger model has less conflicts due to co-training, it also has less benefits in terms of multi-task generalization.

To balance Pareto efficiency and generalization and achieve the best of both worlds, we propose a method of adding \textit{under-parameterized self-auxiliaries} for multi-task deep learning models. Our proposed method regularizes the learning of shared parameters through additional small towers on the same tasks. It is task-agnostic and can be combined with existing multi-task learning algorithms. 

The paper is organized as follows. In Section \ref{related_work}, we discuss related works in multi-objective optimization for multi-task learning and state-of-the-art methods for improving Pareto efficiency. In Section \ref{understand_pf}, we share our empirical understanding on the Pareto frontier and parameterization trade-offs for multi-task learning. Motivated by such understanding, we propose our method of under-parameterized self-auxiliaries in Section \ref{method}. Finally, we conduct experiments on three benchmark datasets in Section \ref{experiments}, and demonstrate the effectiveness of the proposed method on both regression and classification tasks in real world multi-task applications.

\section{Related Work}
\label{related_work}
\textbf{Multi-task learning as multi-objective optimization.} 
Given the model space, minimizing task losses in multi-task learning can be viewed a multi-objective optimization problem \citep{sener2018multi}. The notion of Pareto optimality and Pareto efficiency was first proposed and studied in multi-objective optimization theories \citep{sawaragi1985theory}. In addition to linear weighting of task losses which are commonly used for multi-task learning problems, examples of other multi-objective optimization methods \citep{miettinen2012nonlinear} include constraint methods, goal programming \citep{jones2010practical}, exponential weighted sum \citep{yu1974cone, athan1996note}, population methods \citep{schaffer1985multiple}, preference elicitation \citep{conitzer2009eliciting}, among many more \citep{kochenderfer2019algorithms}. There is also research on multi-objective optimization methods where the objectives are nonconvex \citep{pardalos2017non} or the Pareto frontier is nonconvex \citep{ghane2015new}. 

Despite the close relationship between multi-task learning and multi-objective optimization, there also exist gaps between them. For example, multi-objective optimization barely looks into the generalization and optimization issue for nonconvex optimization of deep neural networks \citep{zhang2016understanding}, which is a main challenge for multi-task learning problems \citep{chen2017gradnorm}. An example of the recent works \citep{sener2018multi, lin2019pareto} toward bridging this gap is the application of multiple-gradient descent algorithm \citep{desideri2012multiple} to multi-task learning, which is a gradient-based multi-objective optimization method.

Inspired by these explorations, our work starts by empirically understanding the Pareto frontiers of multi-task learning problems from a multi-objective optimization perspective. We find that trade-off exists with different parameterizations as they lead to different training and generalization difficulties. This is rarely discussed in multi-objective optimization literature. Based on our findings, we then propose a simple yet effective treatment to balance the Pareto efficiency improvements from over-parameterization and the generalization benefits from learning multiple tasks jointly.  

\textbf{Improving Pareto efficiency for multi-task deep learning.} Recent research on improving Pareto efficiency for multi-task deep learning can be mainly grouped into three lines of efforts. The first line aims at improved model architecture for more flexible parameter sharing to deal with task conflicts. Examples include soft parameter-sharing architectures that encourage more sharing for similar tasks and less for conflicting tasks \citep{misra2016cross, hashimoto2016joint, ma2018modeling}, adaptively and dynamically deciding which layers to share for which tasks during the training process \citep{lu2017fully, vandenhende2019branched}, or on a higher level, deciding which tasks should be learned together \citep{standley2019tasks}. The second line of research is on optimization algorithms that improve optimization and land on better local optima on the nonconvex loss surface. These works mainly focus on adaptive linear weighting approaches \citep{kendall2018multi, chen2017gradnorm, yu2020gradient, Dosovitskiy2020You} that find better solutions than the vanilla linear weighting method. The third line of research lies in adding auxiliary tasks to improve the performance of main tasks. Auxiliary tasks relate to the main tasks so that jointly predicting them will benefit the main tasks. This line of research has been widely applied to computer vision \citep{zhang2014facial}, natural language processing \citep{arik2017deep} and information retrieval \citep{liu2015representation}. If related tasks are unavailable, auxiliary tasks can also be constructed using adversarial loss \citep{ganin2014unsupervised}, predicting inputs or past labels \citep{caruana1997promoting, caruana1997multitask}, pseudo-task augmentation \citep{meyerson2018pseudo} or learning representations \citep{rei2017semi}. 

Our method of under-parameterized self-auxiliaries can be viewed as along the line of auxiliary tasks. But unlike the auxiliary tasks studied in existing literature, our method does not require any specific domain knowledge on designing auxiliary tasks. The auxiliary tasks in our case is self-auxiliary, in the sense that they are learning the same tasks but with different parameterizations. A similar idea along this line is knowledge distillation \citep{hinton2015distilling, anil2018large}, but our method differs from distillation and existing distillation works in multi-task learning \citep{liu2019improving} in two significant ways: (1) instead of having a small network (student) to learn the predictions of a bigger network (teacher), we let both networks learn exactly the same task; (2) instead of first training the bigger network and then the smaller network, we co-train both networks which share the same labels and learned representations.

\section{Understanding the Parameterization Effect for Multi-Task Learning}
\label{understand_pf}
Suppose there are $T$ tasks sharing an input space $\mathcal{X}$. Each task has its own task space ${\{\Y^t\}}_{i=1}^T$ . A dataset of n i.i.d. examples from the input and task spaces is given by ${\{(x_i, y^1_i, ... ,y^T_i)\}}_{i=1}^n$, where $y^t_i$ is the label of the $t$-th task for example $i$. We assume a multi-task model parameterized by $\theta\in\Theta$. $\theta = (\theta_{sh}, \theta_1,...,\theta_T)$ includes shared-parameters $\theta_{sh}$ and task-specific parameters $\theta_1,...,\theta_T$. Let $f_t(\cdot, \cdot): \mathcal{X}\times\Theta \rightarrow \mathcal{Y}^t$ be the model function and $\L_t(\cdot, \cdot): \mathcal{Y}^t\times\mathcal{Y}^t \rightarrow \mathbb{R}^{+}$ be the loss function for the $t$-th task. This formulation also includes the more general multi-task learning setting where different tasks have different inputs, in which case $x_i = (x^1_i, ... ,x^T_i )^T$ where $x^t_i$ is the input of the $t$-th task for example $i$.

Let $\Lhat_t(\theta) \coloneqq \frac{1}{n}\sum_{i=1}^n \L_t(f_t(x_i;\theta_{sh}, \theta_t), y^t_i)$ be the empirical loss for task $t$, where we drop the dependency on $x$ and $y$ for ease of notation. The optimization for multi-task learning can then be formulated as a joint optimization of a vector-valued loss function:
\begin{equation}
\label{eqn:3.1}
\min_{\theta} (\Lhat_1(\theta), ..., \Lhat_T(\theta))^\top.
\end{equation}


It is unlikely that a single $\theta$ optimizes all objectives simultaneously. The solution to (\ref{eqn:3.1}) is therefore a set of points which represent different trade-off preferences. More formally, solution $\theta_a$ is said to dominate solution $\theta_b$ if $\Lhat_t(\theta_a) \leq \Lhat_t(\theta_b), \forall t $ and there exist at least one task $j$ such that the inequality is strict. A solution $\theta$ is called Pareto optimal if there is no solution $\theta' \neq \theta$ such that $\theta'$ dominates $\theta$. Pareto frontier is the set of all Pareto optimal solutions.


For linear weighting method, the minimization objective is a scalarization of the empirical loss vector $\Lhat(\theta)\coloneqq\sum_{t=1}^T w_t\Lhat_t(\theta)$, where $\{w_t\}_{t\in\{1,...,T\}}$ are weights for individual tasks. However, linear weighting method can only obtain solutions to the above minimization problem in the convex region of the Pareto frontier. Will this be a challenge to multi-task deep learning models? In addition, over-parameterization for single-task deep learning models is almost always desirable \citep{belkin2019reconciling}. Is this also the case with multi-task deep learning models?


\textbf{Benefits from large models.} We first look into the convexity of the Pareto frontier of a multi-task learning model, which determines whether the use of linear weighting method is legitimate. When all objectives are convex in their respective parameters, the Pareto frontier is guaranteed to be convex. 
\begin{pro}
\label{prop1}
Suppose $\L_t(\theta)$ is convex and continuous in $\theta$ for all tasks $t \in \{1,...,T\}$ and $\Theta$ is convex. Then the Pareto frontier of $(\Lhat_1(\theta),..., \Lhat_T(\theta))^\top$ in problem (\ref{eqn:3.1}) is convex.
\end{pro}
\addtocounter{pro}{-1}

\begin{wrapfigure}{r}{0.32\linewidth}
\centering
\includegraphics[width=0.32\textwidth]{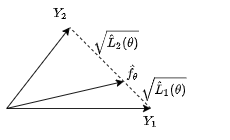}
\caption{Training loss trade-off for over-parameterized fully-shared multi-task model.}
\label{fig3.2}
\end{wrapfigure}
\vspace{-5pt}
See Appendix \ref{sec:proof} for proof. When some or all objectives are nonconvex, it is unlikely that the Pareto frontier remains convex. But there is still something to say about the shape of the Pareto frontier. We start with the case where all tasks are forced to share all parameters, i.e. $\theta=\theta_{sh}$. Consider square loss for regression tasks for ease of visualization. Let $Y_1=(y_1^1,...,y_n^1)^\top, Y_2=(y_1^2,...,y_n^2)^\top$ and $\hat{f_{\theta}}=(\hat{f}(x_1),...,\hat{f}(x_n))^\top$. Then $\Lhat_t(\theta)$ is the squared Euclidean distances between $Y_t $ and $\hat{f_{\theta}}$. Now assume that $f_{\theta}$ is over-parameterized enough so that $\hat{f_{\theta}}$ is able to \textit{fully} populate the $n$-dimensional space. In other words, for any $Y=(y_1,...,y_n)$, there exists $\theta \in \Theta$ such that $\L(f_{\theta}(x_i), y_i)=0, \forall i$. In this case, it is obvious to see that Pareto optimality is obtained when $\hat{f_{\theta}}$ is a linear combination of $Y_1$ and $Y_2$, as shown in Figure \ref{fig3.2}, where $\sqrt{\Lhat_1(\theta)} + \sqrt{\Lhat_2(\theta)} = \ltwonorm{Y_1-Y_2}$. Therefore the Pareto frontier is convex. Similar arguments can be made with any convex loss function.

When task-specific parameters are allowed, over-parameterized multi-task models can achieve zero training loss. In this case the Pareto frontier is an orthant, which is also convex. Note that the the Pareto frontier discussed above is the optimal training loss value considering all possible $f_{\theta} \in \mathcal{H}$, without considering optimization error or generalization error. However this provides some justification for using linear weighting methods for over-parameterized multi-task learning models. In addition, from the perspective of multi-objective optimization, task conflicts are reflected as the trade-off among task objectives over a constrained solution space. Over-parameterization enables better handling of task conflicts. 

\textbf{Challenges inherent in large models.} In order to understand whether bigger models always lead to better multi-task performance, We perform a series of studies on synthetic datasets. Similar to the setup in \cite{finn2017model} and \cite{ma2018modeling}, we generate a multi-task dataset and define each task as a regression from the input to the output of a combination of sine waves. To introduce task conflicts together with task correlation, we let the two tasks share a small subset of frequencies. A shared-bottom model with fully connected ReLU layers is used. A full description of the synthetic dataset and model architectures is available in Appendix \ref{sec:sythetic-dataset-description}. In order to observe the parameterization effect, we plot the Pareto frontier of the test losses with different model capacities. Figure \ref{fig:4.2a} shows that with more hidden layers added to task-specific towers, the Pareto frontier first improves and then deteriorates. We also observe similar trends when only increasing the shared layers or both shared layers and task-specific layers, and when varying network width instead of depth. 

The intriguing observation motivates us to better understand the parameterization effect of multi-task models. Multi-objective optimization theory suggests that enough mode capacity is needed to be able to deal with task conflicts. However, treating a multi-task learning model simply as a multi-objective optimization problem only sees one side of the game. Multi-task learning is a more general problem than multi-objective optimization as it leverages parameter sharing and inductive transfer \citep{baxter2000model} to benefit generalization. Over-parameterization intuitively undermines the benefit of sharing, which may hurt multi-task generalization and eventually backfire. Thus, the parameterization effect for MTL is not singular. Instead, it is a trade-off between efficiency and generalization. 

It is worth pointing out that the relation between efficiency and generalization has been discussed in single-task deep learning models. Recent studies show an intriguing double descent generalization curve \citep{belkin2019reconciling} that subsumes the traditional U-shaped bias-variance trade-off curve, which shows that increasing model capacity beyond the point of interpolation results in improved generalization. For multi-task learning, generalization is also a result of how different tasks are sharing representation on top of individual task-specific parameterizations. Therefore the relation between Pareto efficiency and generalization could be more intricate than single-task cases. 

To summarize our insights, for a multi-task learning model, small models benefit from good multi-task generalization but hurts Pareto efficiency; big models theoretically have better Pareto efficiency but could suffer from loss of generalization. This motivates us to design a treatment towards achieving the best of both worlds. We discuss our proposal in the next section. 





\section{Under-Parameterized Self-Auxiliaries for Multi-Task Learning}
\label{method}

\begin{figure}
    \begin{subfigure}[b]{0.33\textwidth}
        \centering
        \includegraphics[width=\textwidth]{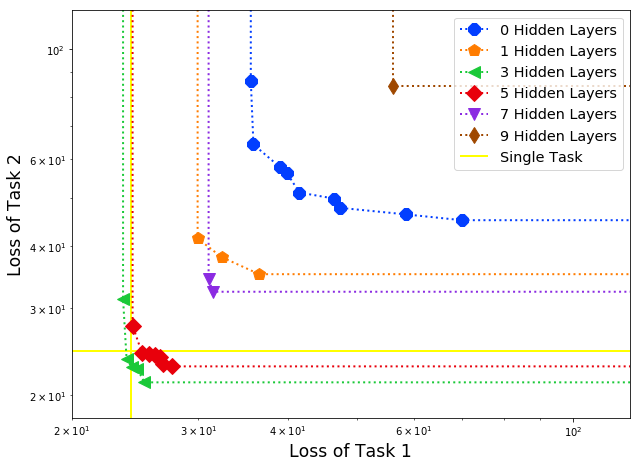}
        \caption{Baseline.}
        \label{fig:4.2a}
    \end{subfigure}
    \centering
    \begin{subfigure}[b]{0.32\textwidth}
        \centering
        \includegraphics[width=\textwidth]{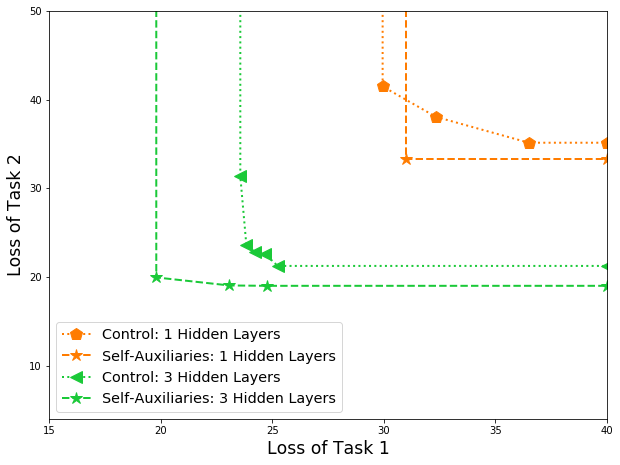}
        \caption{1-3 hidden layers.}
        \label{fig:4.2b}
    \end{subfigure}
    \begin{subfigure}[b]{0.32\textwidth}
        \centering
        \includegraphics[width=\textwidth]{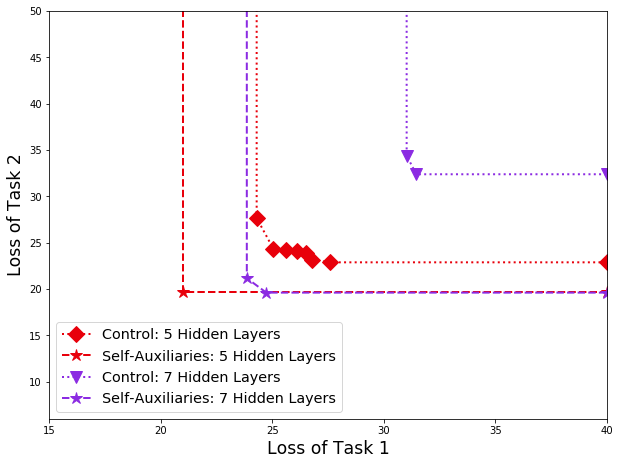}
        \caption{5-7 hidden layers.}
        \label{fig:4.2c}
    \end{subfigure}
   \caption{Pareto frontiers on synthetic data. (a): Baseline Pareto frontiers with increasing model capacity. The best single task performance across all models is also reported. (b)-(c): Comparison of our method with baseline on different model capacities.}
  \label{fig4.2}
\end{figure}

We now introduce our proposed method of \textit{under-parameterized self-auxiliaries} for multi-task deep learning. The most popular architecture for multi-task learning problems consists of a shared representation across all tasks together with separate task-specific functions. The model family $\mathcal{H}_t$ for each task $t$ is constrained to be: 
\begin{equation}
\label{eqn:4.1}
f_t(x;\theta_{sh}, \theta_t) = f_t(h(x;\theta_{sh}); \theta_t), \forall t,
\end{equation}
where $h(\cdot;\cdot): \mathcal{X}\times\Theta \rightarrow \mathcal{R}^M$ represents the shared representation.

Now we construct a self-auxiliary tower for every task with the same task labels and a different parameterization:
\begin{equation}
\label{eqn:4.2}
f_t^{a}(x;\theta_{sh}, \theta_t^{a}) = f_t(h(x;\theta_{sh}); \theta_t^{a}), \forall t, 
\end{equation}

where the superscript stands for auxiliary. Inspired by the insights from Section \ref{understand_pf}, we let the auxiliary towers parameterized by $\theta_t^{a}$ to be small enough. With the under-parameterized self-auxiliaries, the empirical loss is then defined as 
\begin{equation}
\label{eqn:4.3}
\Lhat(\theta)=\sum_{t=1}^T w_t (\Lhat_t(\theta_{sh}, \theta_t) + \gamma\Lhat_t(\theta_{sh}, \theta_t^a))
\end{equation}

\begin{wrapfigure}{r}{0.4\linewidth}
\centering
\includegraphics[width=0.4\textwidth]{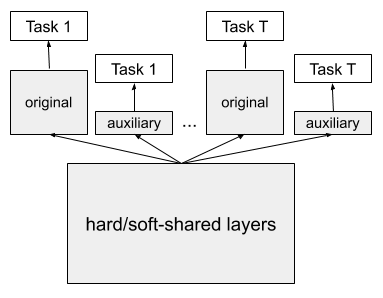}
\caption{An illustration of under-parameterized self-auxiliaries for multi-task learning.}
\label{fig4.1}
\vspace{-5mm}
\end{wrapfigure}

where $w_t$ is the weight for task $t$, $\Lhat_t(\theta_{sh}, \theta_t^a) = \frac{1}{n}\sum_{i=1}^n \L_t(f_t(x_i;\theta_{sh}, \theta_t^a), y^t_i)$ is the loss for task $t$'s self-auxiliary, and $\gamma>0$ controls the weight of the auxiliary loss. At inference time, the self-auxiliaries are discarded and \textit{only} $f_t(x;\theta_{sh}, \theta_t)$ is used as task $t$'s prediction.

Figure \ref{fig4.1} illustrates the proposed method. Self-auxiliaries effectively double the number of tasks. However, because they are under-parameterized, only a small number of additional parameters is introduced. And because self-auxiliaries are only used in training as an additional loss term, they do not incur any extra cost at serving time. 
We experiment with under-parameterized self-auxiliaries on the synthetic example in Section \ref{understand_pf}. As shown in Figure \ref{fig:4.2b} and \ref{fig:4.2c}, under-parameterized self-auxiliaries significantly improve the Pareto frontier on the test dataset under all levels of parameterization. And the improvement is larger with larger models.



The fact that under-parameterized self-auxiliaries improve Pareto efficiency for multi-task learning is not surprising to us. From the insights in Section \ref{understand_pf}, over-parameterized multi-task learning model undermines the benefit of shared representations. By co-training the same tasks with under-parameterized small towers, the shared part of the multi-task model $h(x;\theta_{sh})$ is forced to learn a shared representation that explains both sets of tasks as much as possible. In other words, the self-auxiliaries act as regularization. However, instead of adding explicit constraints on model capacity or prior assumptions on the weight distribution, the self-auxiliaries \textit{implicitly} regularizes the multi-task training dynamics. As a result, sharing happens in bottom levels as much as possible. And the original task towers effectively have more capacity to learn the task specifics and handle conflicts. In this way, even if the original task towers have the exact same parameterization as the case without self-auxiliaries, they are able to generalize better because of improved learning of the shared parameters.  

The architectures for the under-parameterized self-auxiliaries are flexible. The general guideline is that any tower $f_t^{a}(x;\theta_{sh}, \theta_t^{a})$ that is significantly smaller than the original tower $f_t(x;\theta_{sh}, \theta_t)$ would work. In this sense our approach is basically model agnostic as it is general, adaptive and can be applied to any model architecture. For example, one can simply use a single fully-connected layer over the shared representation $h(x;\theta_{sh})$ as the self-auxiliary tower (Figure \ref{fig:4.3a}). If the original model is relatively small or dimension of the shared layer output $M$ is big, we can further reduce the parameterization of self-auxiliaries through average pooling (Figure \ref{fig:4.3b}) of the shared representation. For multi-class classification tasks, the final layer is a softmax layer with the size equal to the number of classes $C$. In this case, a single fully-connected layer as self-auxiliaries introduce $CM$ additional parameters, which could be a considerable amount if both $C$ and $M$ are large as in many multi-task applications. We can instead let the self-auxiliaries be a two-layer tower with a bottleneck layer of size $b \ll M, C$ (Figure \ref{fig:4.3c}), in this case the number of additional parameters will be $\mathcal{O}(\max(C, M))$ instead of $\mathcal{O}(CM)$.

It is worth noting that, obviously, the proposed method of under-parameterized self-auxiliaries for multi-task learning can be combined with almost all existing multi-task learning algorithms including uncertainty weighting \citep{kendall2018multi}, gradient surgery \citep{yu2020gradient} and multi-objective optimization algorithms \citep{sener2018multi, lin2019pareto}.

\begin{figure}
    \centering
    \begin{subfigure}[b]{0.31\textwidth}
        \centering
        \includegraphics[width=\textwidth]{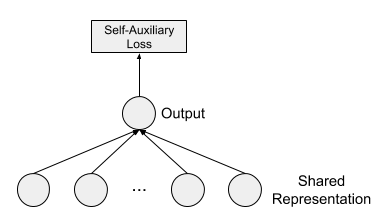}
        \caption{Single fully-connected layer.}
        \label{fig:4.3a}
    \end{subfigure}
    \begin{subfigure}[b]{0.31\textwidth}
        \centering
        \includegraphics[width=\textwidth]{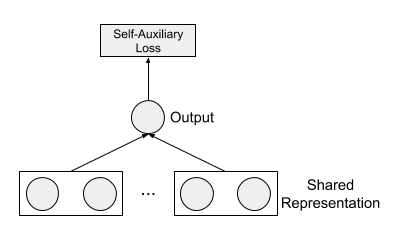}
        \caption{Average pooling.}
        \label{fig:4.3b}
    \end{subfigure}
    \begin{subfigure}[b]{0.31\textwidth}
        \centering
        \includegraphics[width=\textwidth]{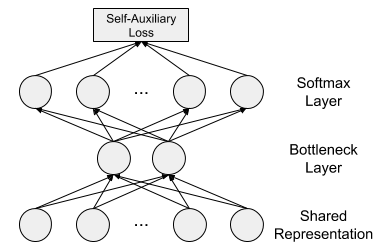}
        \caption{Bottleneck layer.}
        \label{fig:4.3c}
    \end{subfigure}
   \caption{Example architectures for under-parameterized self-auxiliaries. (a): Single fully-connected layer. (b): Single layer with average pooling. (c): Two-layer tower with bottleneck layer.}
  \label{fig4.3}
\end{figure}

\section{Experiments}
\label{experiments}
We evaluate the proposed method of under-parameterized self-auxiliaries on a number of multi-task applications with classification and regression tasks. To demonstrate the effectiveness and robustness of our method, we conduct experiments on benchmark datasets from two different areas, image classification and recommendation systems. We compare our method with state-of-the-art model-agnostic methods for multi-task learning. 

The baselines we compare against are: (1) \textbf{single task baseline(ST)}: learning each task separately; (2) \textbf{linear weighting (MTL)}: linear scalarization method $\Lhat(\theta)=\sum_{t=1}^T w_t\Lhat_t(\theta)$ with the weights varying in the simplex $\{w=(w_1,...w_T)|\sum_{t=1}^T w_t = 1, w_t\geq 0, \forall t\}$; (3) \textbf{uncertainty weighting (Uncertainty)}: learning uncertainty of the tasks which are used as loss weights \citep{kendall2018multi}; (4) \textbf{Multiple-gradient descent algorithm (MGDA-UB)}: using a modified multiple-gradient descent algorithm from multi-objective optimization \citep{sener2018multi}; (5) \textbf{gradient surgery (PCGrad)}: a gradient projection method for mitigating task conflicts \citep{yu2020gradient}. 


\subsection{MultiMNIST and MultiFashionMNIST}
\label{sec:5.1}
We first conduct experiments on multiple image classification tasks. We extend MNIST \citep{lecun1998gradient} and FashionMNIST \citep{xiao2017fashion} to a multi-task setup similar to \cite{sener2018multi} and \cite{lin2019pareto}. For MultiMNIST, two $32\times32$ images are chosen at random from the MNIST dataset. Then one is put at the top-left corner and the other is at bottom-right, overlapping each other with a vertical and horizontal stride as 4 pixels. MultiFashionMNIST is constructed in the same way with the images from the FashionMNIST dataset. For each dataset, the multi-task learning problem is to classify the item on the top-left (task 1) and bottom-right (task 2) for each combined image. 

We adopt the LeNet architecture \citep{lecun1998gradient} to build the multi-task model for both applications, with convolutional layers as shared layers and ReLU layers as task-specific layers. We use the bottleneck architecture described in Figure \ref{fig:4.3a} as the self-auxiliary towers. To demonstrate the effectiveness of under-parameterized self-auxiliaries on different model sizes, we compare our method with baseline methods on three model architectures, with increasing number of shared hidden layers and task-specific hidden layers. Details on the architecture and hyperparameters for the main models and self-auxiliaries can be found in Appendix \ref{sec:mnist}. For every method and every model size, we first perform 1000 runs for hyperparameters tuning, and then do another 1000 runs with the selected hyperparameters to generate the Pareto frontier. We visualize the results on task 1 and task 2 test accuracy in Figure \ref{fig5.1}. Our method achieves similar performance compared with the best baseline method for small models (Figures \ref{fig:5.1a}, \ref{fig:5.1d}), and better than other baselines for medium (Figures \ref{fig:5.1b}, \ref{fig:5.1e}) and large models (Figures \ref{fig:5.1c}, \ref{fig:5.1f}). We also observe that the larger the model, the more improvement our method has over baseline methods. Additional details and results of the experiments can be found in Appendix \ref{sec:mnist}.

\begin{figure}
    \centering
    \begin{subfigure}[b]{0.315\textwidth}
        \centering
        \includegraphics[width=\textwidth]{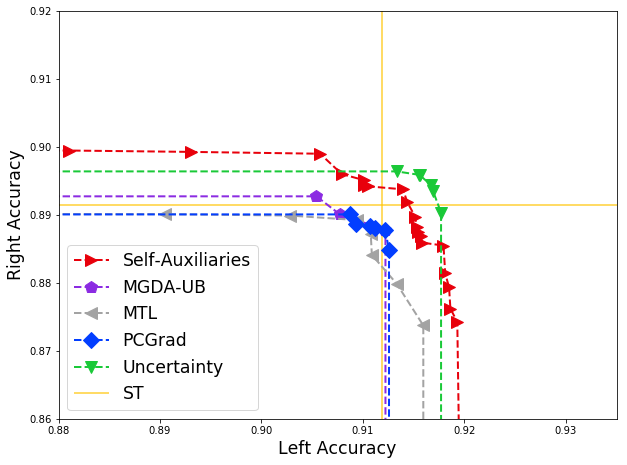}
        \caption{MultiMNIST small model.}
        \label{fig:5.1a}
    \end{subfigure}
    \begin{subfigure}[b]{0.32\textwidth}
        \centering
        \includegraphics[width=\textwidth]{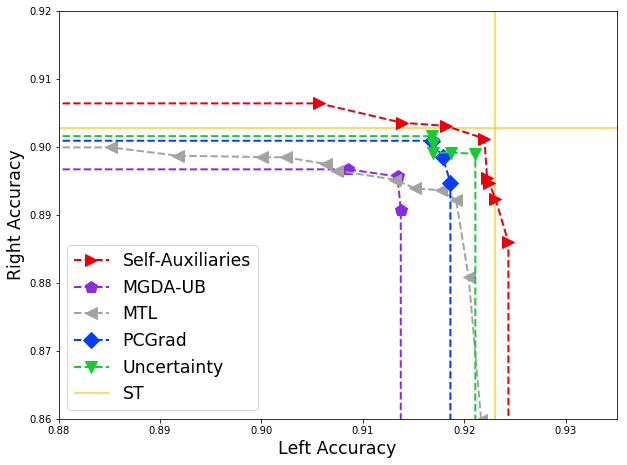}
        \caption{MultiMNIST medium model.}
        \label{fig:5.1b}
    \end{subfigure}
    \begin{subfigure}[b]{0.32\textwidth}
        \centering
        \includegraphics[width=\textwidth]{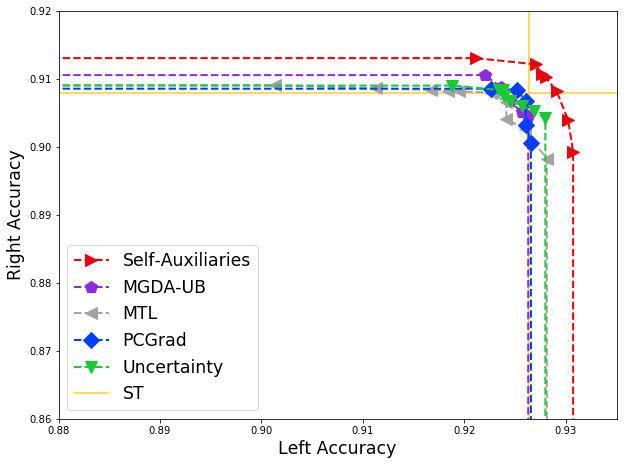}
        \caption{MultiMNIST large model.}
        \label{fig:5.1c}
    \end{subfigure}
    \begin{subfigure}[b]{0.32\textwidth}
        \centering
        \includegraphics[width=\textwidth]{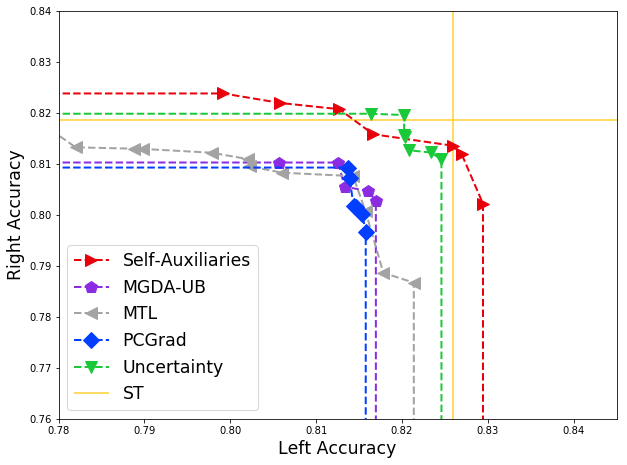}
        \caption{MultiFashion small model.}
        \label{fig:5.1d}
    \end{subfigure}
    \begin{subfigure}[b]{0.32\textwidth}
        \centering
        \includegraphics[width=\textwidth]{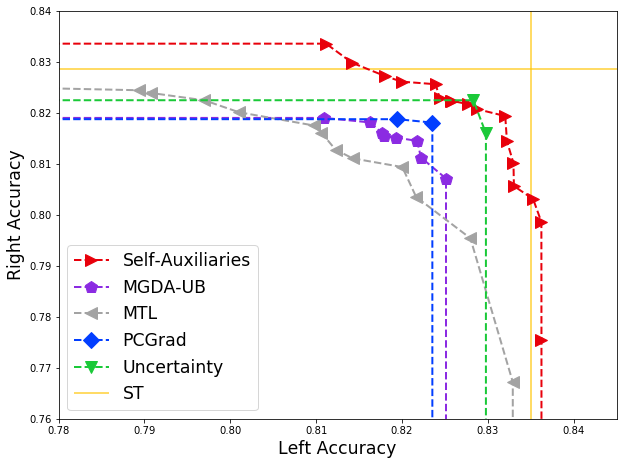}
        \caption{MultiFashion medium model.}
        \label{fig:5.1e}
    \end{subfigure}
    \begin{subfigure}[b]{0.32\textwidth}
        \centering
        \includegraphics[width=\textwidth]{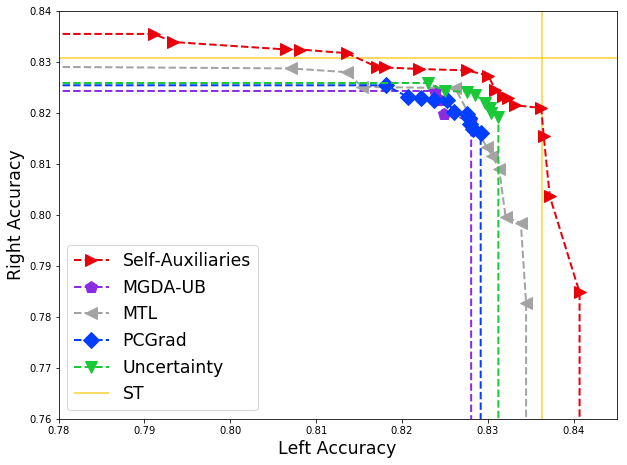}
        \caption{MultiFashion large model.}
        \label{fig:5.1f}
    \end{subfigure}
   \caption{Experiment results on MultiMNIST/MultiFashion datasets with different model capacities.}
  \label{fig5.1}
\end{figure}

\subsection{MovieLens}

We further evaluate our method on another multi-task application of moive recommendatino with the tasks being a mixture of regression and classification tasks. To this end, we use the MovieLens 1M dataset \citep{harper2015movielens}, which records 1 million ratings from 6000 users on 4000 movies. For every user and movie pair, we construct the following two tasks: A binary classification task to predict whether the user watches the movie (task 1); and a regression task to predict the user's rating (1-5) on the movie as a float value. The design of the tasks as well as the model architecture is similar to what is described in a real-world large-scale recommendation system \citep{covington2016deep}. 

Specifically, we adopt a shared-bottom model architecture with shared layers and task-specific layers of ReLU activation \citep{caruana1997multitask}. To further understand how different parameterizations of self-auxiliaries affect model performance, we also experiment with average pooling (Figure \ref{fig:4.3b}) with different pool sizes on the last shared hidden layer as the input for the self-auxiliary towers. The performance of the tasks is measured in error rate for watch prediction and mean squared error (MSE) for rating prediction on the test dataset. Details on data processing and model architecture can be found in Appendix \ref{sec:movielens}.

For each baseline method and our method, we perform 1000 runs to search for the best hyper-parameters in terms of learning rate and model architecture. Then we run each method with its best hyper-parameter setup for another 1000 runs to generate final results. Figure \ref{fig:5.2a} shows the Pareto frontier for the two tasks, indicating that our method with average pooing significantly improves the performance on both tasks. To understand the effectiveness of average pooling, Figure \ref{fig:5.2b} shows the performance of our methods with different input dimensions for the self-auxiliary tower after average pooling. We can see that by reducing the parameterization of self-auxiliary towers, we can further improve the performance of our method.
\begin{figure}
    \centering
    \begin{subfigure}[b]{0.34\textwidth}
        \centering
        \includegraphics[width=\textwidth]{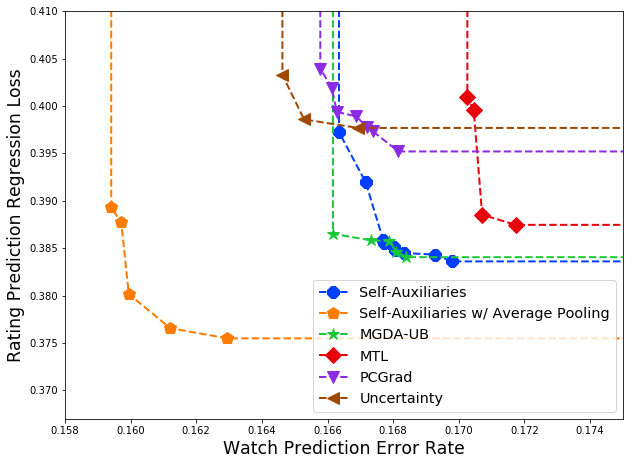}
        \caption{Self-auxiliaries vs. baselines.}
        \label{fig:5.2a}
    \end{subfigure}
    \hspace{1.5cm}
    \begin{subfigure}[b]{0.34\textwidth}
        \centering
        \includegraphics[width=\textwidth]{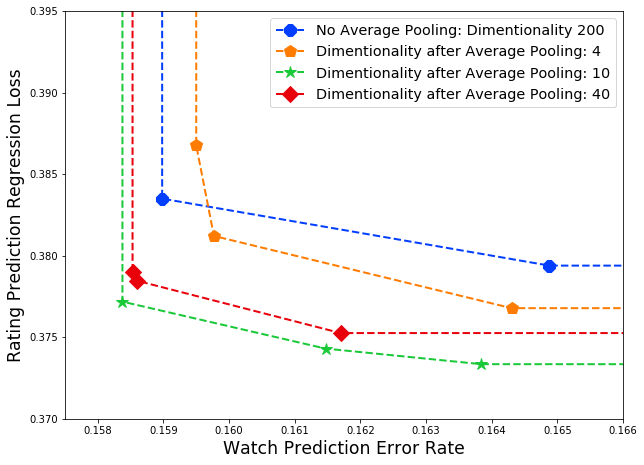}
        \caption{Effect of average pooling.}
        \label{fig:5.2b}
    \end{subfigure}
   \caption{Experiment results on MovieLens dataset.}
  \label{fig5.2}
\end{figure}

\subsection{Discussion}
We demonstrate the effectiveness of under-parameterized self-auxiliaries in three multi-task datasets covering different applications. Our proposed method is able to significantly improve Pareto efficiency in real-world multi-task problems, with a negligible increase in the number of parameters introduced at training time and no additional cost at inference time. It works well on different model architectures and different types of tasks. 

In addition, we observe larger benefit of our method with larger model architectures. This confirms our insights from Section \ref{method} that under-parameterized self-auxiliaries help achieve a better balance between Pareto efficiency and multi-task generalization through implicit regularization. We also observe that the performance of self-auxiliaries can be improved through further reducing the parameterization of the self-auxiliaries. Techniques such as average pooling and adding bottleneck layers can be treated as hyperparameters in practice to optimize the performance.

\section{Conclusion}
In this paper, we showed that, unlike single-task learning, there is a trade-off between efficiency and generalization introduced by parameterization in multi-task models. We proposed a method of under-parameterized self-auxiliaries for multi-task learning. By introducing auxiliary losses from adding small towers on the same tasks, the model is able to achieve better Pareto efficiency from balancing multi-objective optimization and multi-task learning. Experimental results on synthetic and real datasets demonstrated the effectiveness of under-parameterized self-auxiliaries in a number of multi-task applications.

\section*{Broader Impact}
Multi-task learning provides a cost-effective way to learn multiple targets simultaneously and has wide applications in real-world problems. Our proposed method of under-parameterized self-auxiliaries is a general methodology for improving multi-task learning. It helps save cost and resources such as machine learning infrastructures.

Its ability to achieve a better trade-off between efficiency and generalization for multi-task learning could help many real-world applications where improvement of multiple objectives is desired with limited budget. Examples include detecting different types of scams on social media, or providing satisfying recommendation results to different user groups. 

Multi-task learning, or all deep learning models in general, have the risk of producing biased predictions reflective of the bias in the training data. This can happen if the model architecture or the objectives are not formulated properly to address the potential concerns. Our work is no exception here. Self-auxiliaries provide a task-agnostic way to improve Pareto efficiency for multi-task learning models. If for example the tasks are formulated without considering potential ethical consequences, then simply adding self-auxiliaries to the model would not help reduce the bias. However in other cases where fairness objectives are added or accounted for in the model architecture, our method helps achieve the fairness goal by improving the Pareto efficiency of the corresponding multi-task models.


\bibliography{references}
\bibliographystyle{ACM-Reference-Format}

\newpage
\appendix
\section{Appendix}

\subsection{Proof of Proposition~\ref{prop1}}
\label{sec:proof}

\begin{pro}
\ref{prop1}
Suppose $\L_t(\theta)$ is convex and continuous in $\theta$ for all tasks $t \in \{1,...,T\}$ and $\Theta$ is convex. Then the Pareto frontier of $(\Lhat_1(\theta),..., \Lhat_T(\theta))^\top$ in problem (\ref{eqn:3.1}) is convex.
\end{pro}

\begin{proof}
For ease of presentation, we only show the proof for $T=2$. The proof naturally generalizes to $T>2$. Let $P(\theta)=(\Lhat_1(\theta), \Lhat_2(\theta))$ denote the feasible point for $\theta \in \Theta$. By definition of a convex curve, we only need to show that for any two points $P(\theta_1), P(\theta_2)$ on the Pareto frontier, the line connecting them, i.e. $\lambda P(\theta_1) + (1-\lambda)P(\theta_2), \forall \lambda \in [0,1]$ is \textit{above} the Pareto frontier. Because $\L_t(\theta)$ is convex in $\theta$, $\Lhat_t(\theta)$ is also convex in $\theta$, for $t=1,2$. By convexity, we have 
\begin{align}
\begin{split}
\label{eqn:3.2}
& \Lhat_1(\lambda\theta_1 + (1-\lambda)\theta_2) \leq \lambda \Lhat_1(\theta_1) + (1-\lambda)\Lhat_1(\theta_2), \\ 
& \Lhat_2(\lambda\theta_1 + (1-\lambda)\theta_2) \leq \lambda \Lhat_2(\theta_1) + (1-\lambda)\Lhat_2(\theta_2), 
\end{split}
\end{align}

\begin{figure}[hbtp!]
\centering
\includegraphics[width=0.3\textwidth]{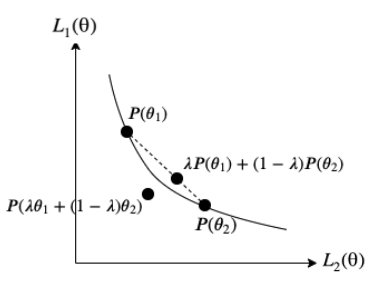}
\caption{Pareto frontier is convex when objectives are convex.}
\label{fig3.1}
\end{figure}

where $\lambda\theta_1 + (1-\lambda)\theta_2 \eqqcolon \theta_0 \in \Theta$ by convexity of $\Theta$. This means that for every point $(\lambda \Lhat_1(\theta_1) + (1-\lambda)\Lhat_1(\theta_2), \lambda \Lhat_2(\theta_1) + (1-\lambda)\Lhat_2(\theta_2)$, there exists a feasible point $(\Lhat_1(\theta_0), \Lhat_2(\theta_0))$ dominating it. As shown in Figure \ref{fig3.1}, this suggests the Pareto frontier between $P(\theta_1)$ and $P(\theta_2)$ is below the line $\lambda P(\theta_1) + (1-\lambda)P(\theta_2), \forall \lambda \in [0,1]$. Therefore the Pareto frontier is convex.
\end{proof}

\subsection{Experiment Details on Synthetic Dataset}
\label{sec:sythetic-dataset-description}

\subsubsection{Details on data generation}
\label{sec:sythetic-dataset-generation}

Inspired by \cite{finn2017model} and \cite{ma2018modeling}, we generate a multi-task dataset and define each task as a regression from the input to the output, with the output being a combination of sine waves. To introduce task conflicts together with task correlation, we let the two tasks share a small subset of frequencies. More specifically, the synthetic dataset is generated as follows:

\begin{enumerate}
	\item[Step 1] Generate the frequency sets used by the two tasks. $W_1 = \{i \in \mathbb{N}: 0 \leq i \leq 29 \text{ or } 50 \leq i \leq 79 \text{ or } 100 \leq i \leq 129 \}$ and $W_2 = \{i \in \mathbb{N}: 25 \leq i \leq 49 \text{ or } 75 \leq i \leq 99 \text{ or } 125 \leq i \leq 149 \}$, so that they have overlapping but mostly different frequencies.
	\item[Step 2] Generate shared inputs. Let input dimension $D=200$ and generate $x_d \sim U[-1/2, 1/2]$ for $1 \leq d \leq D$. 
	\item[Step 3] Generate outputs. Let $x=\sum_{d=1}^D x_d$ and generate $e_1, e_2 \sim N(0,1)$. The labels $y_1$, $y_2$ for the two regression tasks are defined as: 
    	\begin{equation}
        \begin{aligned}
          y_1 &= \sum_{w_1 \in W_1}(w_1 x + 0.2 e_1) \\
          y_2 &= \sum_{w_2 \in W_2}(w_2 x + 0.2 e_2), 
        \end{aligned}
        \end{equation}
    \item[Step 4] Repeat Step 2-3 $n_{\text{train}}=100000$ times to generate training dataset, and $n_{\text{test}}=10000$ times to generate test dataset.
\end{enumerate}

Figure \ref{fig:a.2} shows the shape of the two tasks as a function of $x$. We adopt the shared-bottom model architecture with 2 shared hidden layers of size 250 and 125 each with ReLU activation, with the input as $(x_1, ... ,x_{D})$. For the two task-specific towers for the output $y_1$ and $y_2$, we fix the size of hidden ReLU layers to be 100 and vary the number of hidden layers from 0 to 9. The regression loss is the mean-squared error between the prediction and the true value. The resulting Pareto frontier on the test dataset is shown in Figure \ref{fig:4.2a}.

\begin{figure}[hbtp!]
\centering
\includegraphics[width=0.45\textwidth]{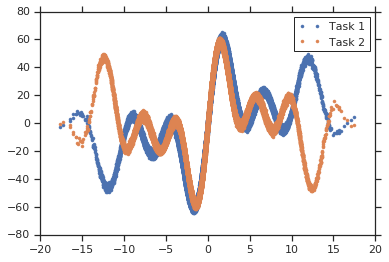}
\caption{An illustration of two tasks as a function of sum of inputs $x=\sum_{d=1}^D x_d$.}
\label{fig:a.2}
\end{figure}

\subsection{Additional experiments on varying network width instead of depth.}
\label{sec:synthetic-width}
In addition to varying depth of the multi-task model as in Figure \ref{fig:4.2a}, we also studied the parameterization effect for multi-task learning models with varying width. On the same dataset in Appendix \ref{sec:sythetic-dataset-generation} and using the same shared-bottom model architecure with 2 shared hidden layers of size 250 and 125 each with ReLU activation, we fix the number of task-specific hidden layers to be 1 and increase the layer width from 100 to 2000. Results in shown in Figure xxx. We draw the same conclusion as in Section \ref{understand_pf} for the parameterization effect and improvements by under-parameterized self-auxiliaries. 

\begin{figure}
    \begin{subfigure}[b]{0.33\textwidth}
        \centering
        \includegraphics[width=\textwidth]{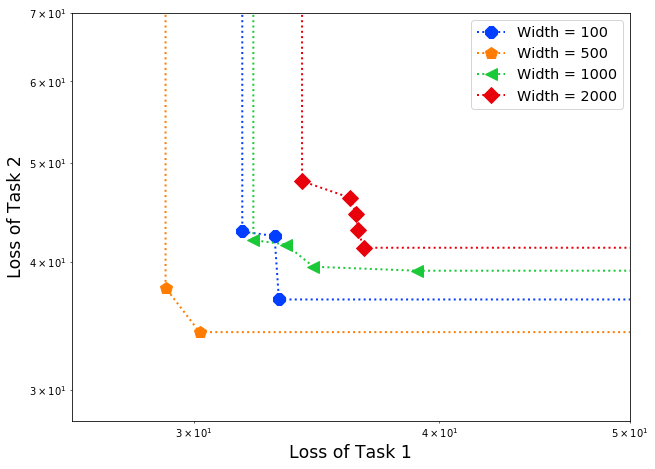}
        \caption{Baseline.}
        \label{fig:a.3a}
    \end{subfigure}
    \centering
    \begin{subfigure}[b]{0.32\textwidth}
        \centering
        \includegraphics[width=\textwidth]{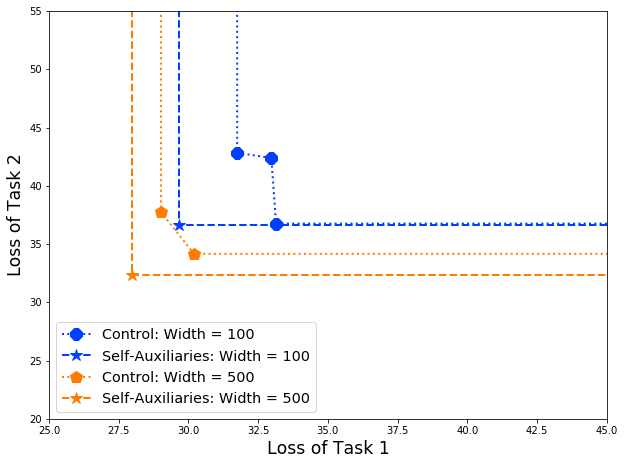}
        \caption{Layer width 100-500.}
        \label{fig:a.3b}
    \end{subfigure}
    \begin{subfigure}[b]{0.32\textwidth}
        \centering
        \includegraphics[width=\textwidth]{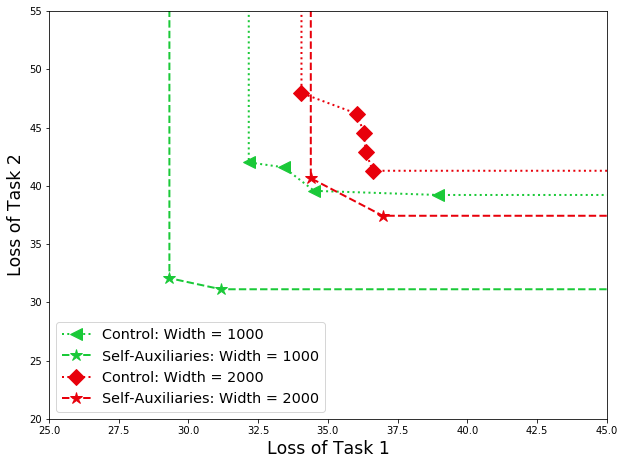}
        \caption{Layer width 1000-2000.}
        \label{fig:a.3c}
    \end{subfigure}
   \caption{Pareto frontiers on synthetic data. (a): Baseline Pareto frontiers with increasing task-specific layer width. The best single task performance across all models is also reported. (b)-(c): Comparison of our method with baseline on different model capacities.}
  \label{figa.3}
\end{figure}

\subsection{Experiment Details on MultiMNIST and MultiFashion}
\label{sec:mnist}

\subsubsection{Model architecture}
Similar to \cite{sener2018multi} and \cite{lin2019pareto}, we construct MultiMNIST and MultiFashion dataset by extending MNIST \citep{lecun1998gradient} and FashionMNIST \citep{xiao2017fashion} to multi-task setups. For each dataset, the multi-task learning problem is to classify the item on the top-left (task 1) and bottom-right (task 2) for each combined image. Figure \ref{figa.3} shows the architecture for the three model architectures used, with increasing number of shared hidden layers and task-specific hidden layers. For each baseline method and each model architecture, the hyperparameters include learning rate and weight for each task if applicable. For self-auxiliaries, the tuning parameters include weight $\gamma$, temperature for the auxiliary tower and width for the bottleneck layer if a bottleneck architecture (Figure \ref{fig:4.3c}) is used. For every method, we first perform 1000 runs for hyperparameters tuning, and then do another 1000 runs on test dataset with the selected hyperparameters to generate the Pareto frontier (Figure \ref{fig5.1}).

\begin{figure}
    \centering
    \begin{subfigure}[b]{0.31\textwidth}
        \centering
        \includegraphics[width=\textwidth]{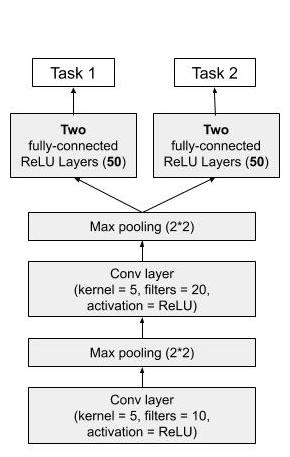}
        \caption{Small model.}
        \label{fig:a.3a}
    \end{subfigure}
    \begin{subfigure}[b]{0.31\textwidth}
        \centering
        \includegraphics[width=\textwidth]{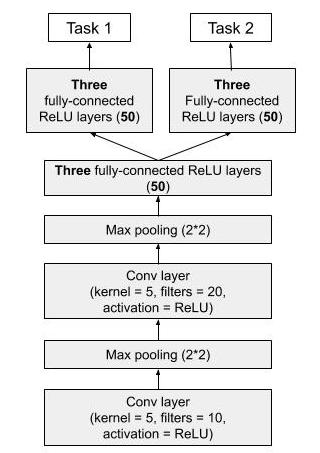}
        \caption{Medium model.}
        \label{fig:a.3b}
    \end{subfigure}
    \begin{subfigure}[b]{0.31\textwidth}
        \centering
        \includegraphics[width=\textwidth]{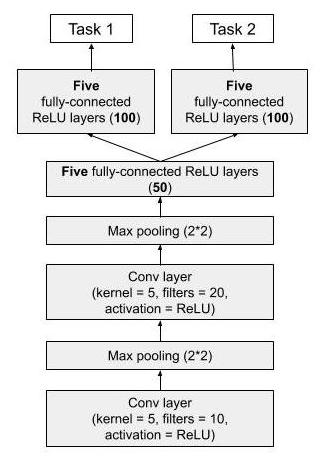}
        \caption{Large model.}
        \label{fig:a.3c}
    \end{subfigure}
   \caption{Model architectures for MultiMNIST and MultiFashion. (a): Small model. (b): Medium model. (c): Large model.}
  \label{figa.3}
\end{figure}

\subsubsection{Numerical results}
In addition to the Pareto frontier reported in Figure \ref{fig5.1}, we also present the numerical results here, by reporting the middle point on the Pareto frontier. Table \ref{tab1} and \ref{tab2} summarize left and right accuracies for different methods on both datasets with different model sizes. Our proposed method of self-auxiliaries achieves on-par performance with the best baseline methods on small models, and outperforms all baseline methods on medium and large models. The results again confirm our observation that the larger the model, the more improvement our method has over baseline methods.

\begin{table}[hbtp]
  \begin{center}
    \begin{tabular}{cc|ccc}
      \hline\hline
      &            & \textbf{Left Accuracy} (\%) & \textbf{Right Accuracy} (\%) \\
      \hline
      \multirow{5}{*}{Small Model}  & MTL                  & 90.95           & 88.92          \\

                            & \textbf{Uncertainty}                  & \textbf{91.56}  & \textbf{89.59}           \\

                            & MGDA-UB                      & 90.09           & 89.09        \\
                            
                            & PCGrad                       & 91.33           & 88.88            \\

                            & Self-Auxiliaries (ours)      & 91.49           & 89.19      \\
                            
      \hline
      \multirow{5}{*}{Medium Model}  & MTL                 & 91.85            & 89.51          \\

                            & Uncertainty                  & 91.86            & 89.92           \\

                            & MGDA-UB                      &  91.44           & 89.62         \\
                            
                            & PCGrad                       & 91.79            & 89.85            \\

                            & \textbf{Self-Auxiliaries (ours)}      & \textbf{92.20}   & \textbf{90.12}      \\
                            
      \hline
      \multirow{5}{*}{Large Model}  & MTL                  & 91.95             & 90.82         \\

                            & Uncertainty                  &  92.45           & 90.68        \\

                            & MGDA-UB                      &  92.45            & 90.67       \\
                            
                            & PCGrad                       &  92.56            & 90.69           \\

                            & \textbf{Self-Auxiliaries (ours)}      & \textbf{92.80}     & \textbf{91.03}      \\

      \hline\hline
    \end{tabular}
    \vspace{1mm}
    \caption{Left and right accuracies for \textbf{MultiMNIST} dataset for different model sizes.\label{tab1}}
  \end{center}
\end{table}

\begin{table}[hbtp]
  \begin{center}
    \begin{tabular}{cc|ccc}
      \hline\hline
      &            & \textbf{Left Accuracy} (\%) & \textbf{Right Accuracy} (\%) \\
      \hline
      \multirow{5}{*}{Small Model}  & MTL                  & 81.26           & 80.96          \\

                            & Uncertainty                  & 82.46          & 81.09           \\

                            & MGDA-UB                      & 81.36           & 81.08        \\
                            
                            & PCGrad                       & 81.42           & 80.94            \\

                            & \textbf{Self-Auxiliaries (ours)}      & \textbf{82.70}           & \textbf{81.19}      \\
                            
      \hline
      \multirow{5}{*}{Medium Model}  & MTL                 & 81.90            & 81.76          \\

                            & Uncertainty                  & 82.83            & 82.25           \\

                            & MGDA-UB                      &  82.38           & 81.97         \\
                            
                            & PCGrad                       & 82.30            & 81.82            \\

                            & \textbf{Self-Auxiliaries (ours)}      & \textbf{83.06}   & \textbf{82.38}      \\
                            
      \hline
      \multirow{5}{*}{Large Model}  & MTL                  &  82.72           & 81.72         \\

                            & Uncertainty                  &  82.99           & 82.25        \\

                            & MGDA-UB                      &  82.68            & 82.06       \\
                            
                            & PCGrad                       &  82.78            & 82.22           \\

                            & \textbf{Self-Auxiliaries (ours)}      & \textbf{83.37}     & \textbf{82.67}      \\

      \hline\hline
    \end{tabular}
    \vspace{1mm}
    \caption{Left and right accuracies for \textbf{MultiFashion} dataset for different model sizes.\label{tab2}}
  \end{center}
\end{table}

\subsubsection{Tuning bottleneck layer for self-auxiliary towers}
When a bottleneck layer (Figure \ref{fig:4.3c}) is adopted as the architecture for self-auxiliaries, the width of the bottleneck layer can be treated as a hyperparameter. This is especially useful for models with moderate sizes so that further reduction of the parameterization for the self-auxiliaries may be desirable. Figure \ref{fig:a.4} shows the Pareto frontier for the self-auxiliaries with different bottleneck sizes on the small model. For medium and large model in our experiments, self-auxiliaries with a single fully-connected layer as in Figure \ref{fig:4.3a} yield the best Pareto frontiers. Therefore no bottleneck layer is needed for those cases.

\begin{figure}
    \centering
    \begin{subfigure}[b]{0.34\textwidth}
        \centering
        \includegraphics[width=\textwidth]{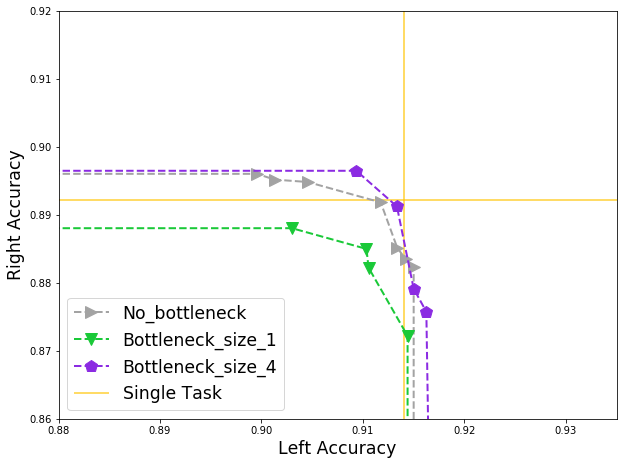}
        \caption{MultiMNIST.}
        \label{fig:a.4a}
    \end{subfigure}
    \hspace{1.5cm}
    \begin{subfigure}[b]{0.34\textwidth}
        \centering
        \includegraphics[width=\textwidth]{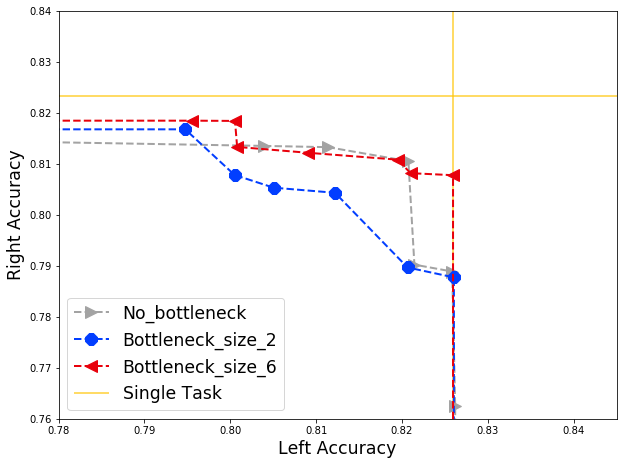}
        \caption{MultiFashion.}
        \label{fig:a.4b}
    \end{subfigure}
   \caption{Effect of bottleneck layer on small models.}
  \label{fig:a.4}
\end{figure}

\subsection{Experiment Details on MovieLens}
\label{sec:movielens} 
We use the MovieLens 1M dataset \citep{harper2015movielens}, which records 1 million ratings (1-5) from 6000 users on 4000 movies. We first augment the dataset by generating negative examples for predicting watch (task 1). We sample user and movie pairs that do not have ratings in the original dataset and treat as negative examples for watch. For every user, the number of un-watched movies is the same as the number of her watched movies. We then sample 1.6 million examples from the augmented dataset as training data, and another 0.2 million examples as test data. The categorical rating values (1-5) are treated as numerical values in prediction rating as a regression problem (task 2).

We adopt a shared-bottom model architecture with shared layers and task-specific layers. Each layer is of size 200 with ReLU activation. For each baseline method, we perform 1000 runs to search for the best learning rate; for our method of self-auxiliaries, we perform 1000 runs to search for the best combination of learning rate, self-auxiliary weight $\gamma$, and self-auxiliary pool length (Figure \ref{fig:4.3b}). In addition to Figure \ref{fig:5.2a}, we also present numerical results in Table \ref{tab3} by reporting the middle point on the Pareto frontier. We can see that with average pooling, our method is able to significantly improve the performance on both tasks compared with baseline methods.

\begin{table}[hbtp]
  \begin{center}
    \begin{tabular}{cc|ccc}
      \hline\hline
      &            & \textbf{Watch Pred. Error Rate} & \textbf{Rating Pred. MSE}\\
      \hline
                            & MTL                          & 0.172           & 0.387          \\
  
                            & Uncertainty                  & 0.165          & 0.399          \\

                            & MGDA-UB                      & 0.168           & 0.385        \\
                            
                            & PCGrad                       & 0.167           & 0.397            \\
                            
                            & Self-Auxiliaries             & 0.168           & 0.385    \\

                            & \textbf{Self-Auxiliaries w/ pooling}      & \textbf{0.161}           & \textbf{0.377}      \\

      \hline\hline
    \end{tabular}
    \vspace{1mm}
    \caption{Numerical results on MovieLens dataset.\label{tab3}}
  \end{center}
\end{table}

\subsection{Additional Experiments}
\label{sec:additional-experiments} 
\subsubsection{Ablation study on the need of "small" vs. "big" towers.}
We also conduct experiments on self-auxiliaries that are not under-parameterized and intuitively introduces less regularization benefits to the learning of shared representation. In particular, we experiment with high-capacity auxiliaries with size identical to that of the original task towers. At inference time, the "big" auxiliary towers are abandoned as before. Figure \ref{figa.4} shows high-capacity self-auxiliaries don't outperform under-parameterized self-auxiliaries on all settings, which confirms the need of "small" towers.

\begin{figure}
    \centering
    \begin{subfigure}[b]{0.315\textwidth}
        \centering
        \includegraphics[width=\textwidth]{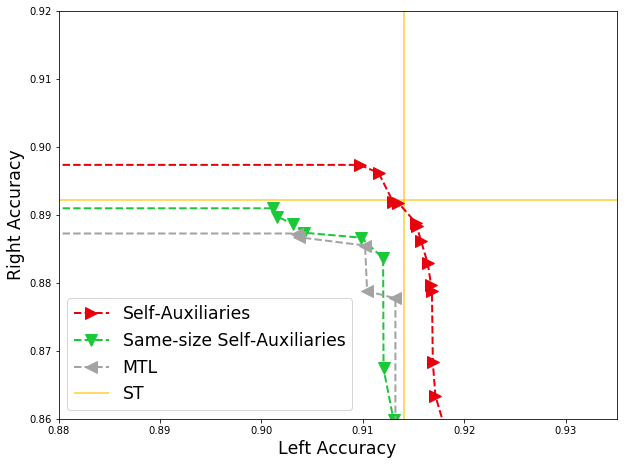}
        \caption{MultiMNIST small model.}
        \label{fig:a.4a}
    \end{subfigure}
    \begin{subfigure}[b]{0.32\textwidth}
        \centering
        \includegraphics[width=\textwidth]{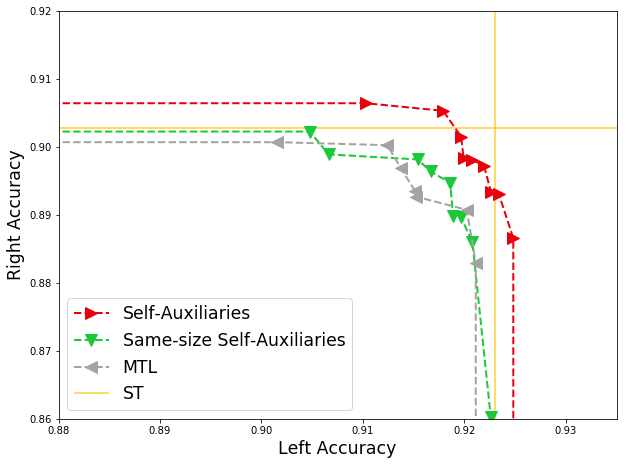}
        \caption{MultiMNIST medium model.}
        \label{fig:a.4b}
    \end{subfigure}
    \begin{subfigure}[b]{0.32\textwidth}
        \centering
        \includegraphics[width=\textwidth]{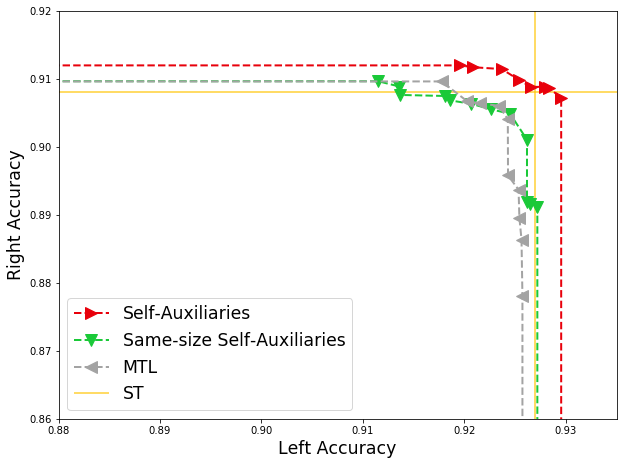}
        \caption{MultiMNIST large model.}
        \label{fig:a.4c}
    \end{subfigure}
    \begin{subfigure}[b]{0.32\textwidth}
        \centering
        \includegraphics[width=\textwidth]{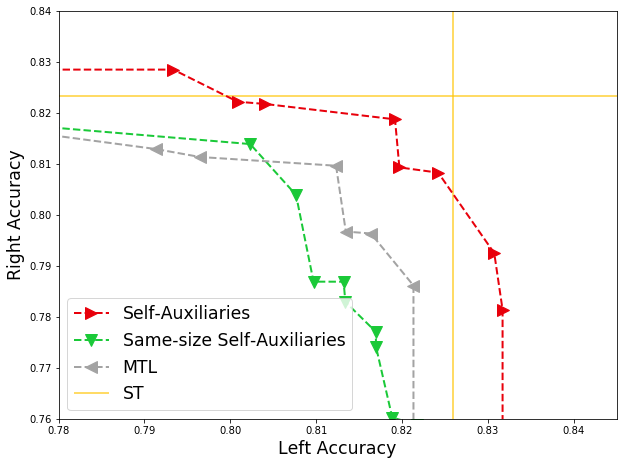}
        \caption{MultiFashion small model.}
        \label{fig:a.4d}
    \end{subfigure}
    \begin{subfigure}[b]{0.32\textwidth}
        \centering
        \includegraphics[width=\textwidth]{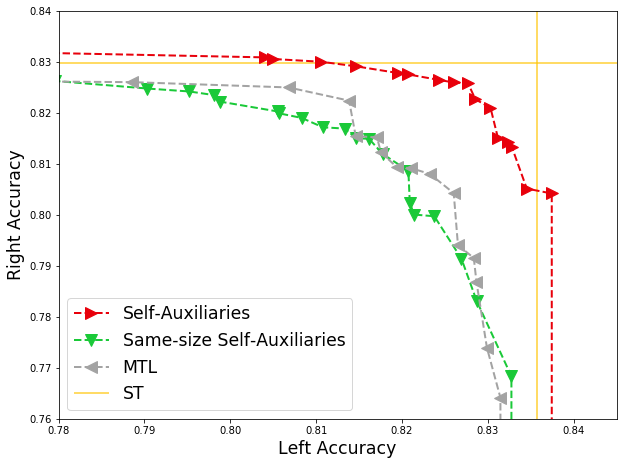}
        \caption{MultiFashion medium model.}
        \label{fig:a.4e}
    \end{subfigure}
    \begin{subfigure}[b]{0.32\textwidth}
        \centering
        \includegraphics[width=\textwidth]{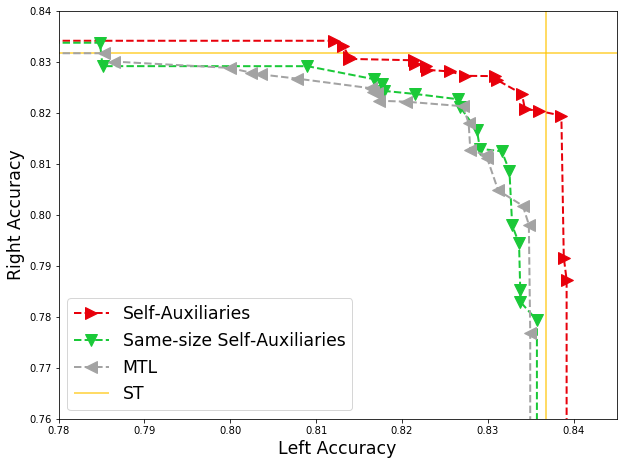}
        \caption{MultiFashion large model.}
        \label{fig:a.4f}
    \end{subfigure}
   \caption{Experiment results on MultiMNIST/MultiFashion datasets with high-capacity auxiliaries with size identical to that of the original task towers (green).}
  \label{figa.4}
\end{figure}

\subsubsection{Self-auxiliary effect on single-task learning.}
To confirm that the benefit of under-parameterized self-auxiliaries comes from better \textit{multi-task} generalization, we further conduct experiments on single-task learning settings. More specifically, we remove one task in the MultiMNIST and MultiFashion dataset and only predicts left or right item as two single task learning problems. Self-auxiliaries are applied at the same layer as in the multi-task case. The results are summarized in Table \ref{tab4} and \ref{tab5}. We also conduct experiments on the regular MNIST dataset with self-auxiliaries (Table \ref{tab6}). We observe that with self-auxiliaries, there is a minor improvement on the small models, but not on medium and large models. Compared with the results on multi-task models in Section \ref{sec:5.1} where self-auxiliaries have larger advantage on larger models, this confirms that our proposed method improves the performance by improving \textit{multi-task} generalization, instead of regularizing each tasks separately.

\begin{table}[hbtp]
  \begin{center}
    \begin{tabular}{c|c|cc}
      \hline\hline
      &            & \textbf{Left Accuracy} (\%) & \textbf{Right Accuracy} (\%) \\
      \hline
      \multirow{3}{*}{Small Model}  & Single Task (ST)  & 91.16           & 88.61          \\

                            & Self-Auxiliaries (ours)      & 91.49           & 89.61      \\
                            \cline{2-4}
                            
                            & \textbf{Difference} & \textbf{+0.58\%} & \textbf{+1.12\%}        \\
                            
      \hline
      \multirow{3}{*}{Medium Model}  & Single Task (ST)  & 91.62           & 89.62          \\

                            & Self-Auxiliaries (ours)      & 91.56           & 89.79      \\
                            \cline{2-4}
                            
                            & \textbf{Difference} & \textbf{-0.07\%} & \textbf{+0.19\%}        \\
                            
      \hline
     \multirow{3}{*}{Large Model}  & Single Task (ST)  & 92.71           & 91.01          \\

                            & Self-Auxiliaries (ours)      & 92.88           & 91.04      \\
                            \cline{2-4}
                            
                            & \textbf{Difference} & \textbf{+0.19\%} & \textbf{+0.03\%}        \\

      \hline\hline
    \end{tabular}
    \vspace{1mm}
    \caption{Self-auxiliaries with single-task learning on \textbf{MultiMNIST} dataset. \label{tab4}}
  \end{center}
\end{table}

\begin{table}[hbtp]
  \begin{center}
    \begin{tabular}{c|c|cc}
      \hline\hline
      &            & \textbf{Left Accuracy} (\%) & \textbf{Right Accuracy} (\%) \\
      \hline
      \multirow{3}{*}{Small Model}  & Single Task (ST)  & 82.14           & 81.95          \\

                            & Self-Auxiliaries (ours)      & 82.91           & 82.70      \\
                            \cline{2-4}
                            
                            & \textbf{Difference} & \textbf{+0.93\%} & \textbf{+0.92\%}        \\
                            
      \hline
      \multirow{3}{*}{Medium Model}  & Single Task (ST)  & 83.13           & 82.97          \\

                            & Self-Auxiliaries (ours)      & 83.10           & 83.04      \\
                            \cline{2-4}
                            
                            & \textbf{Difference} & \textbf{-0.05\%} & \textbf{+0.08\%}        \\
                            
      \hline
     \multirow{3}{*}{Large Model}  & Single Task (ST)  & 83.55           & 82.89         \\

                            & Self-Auxiliaries (ours)      & 83.50           & 83.10      \\
                            \cline{2-4}
                            
                            & \textbf{Difference} & \textbf{-0.07\%} & \textbf{+0.25\%}        \\

      \hline\hline
    \end{tabular}
    \vspace{1mm}
    \caption{Self-auxiliaries with single-task learning on \textbf{MultiFashion} dataset. \label{tab5}}
  \end{center}
\end{table}

\begin{table}[hbtp]
  \begin{center}
    \begin{tabular}{c|ccc}
      \hline\hline
               & Small Model & Medium Model  & Large Model    \\
      \hline
      Single-Task  & 0.9833 $\pm$ 0.0011  & 0.9911 $\pm$ 0.0009   & 0.9918 $\pm$ 0.0009  \\

     Single-Task + Self-Auxiliaries   & 0.9888 $\pm$ 0.0007     & 0.9910 $\pm$ 0.0031  & 0.9920 $\pm$ 0.0011      \\
                            \cline{2-4}
      \hline\hline
    \end{tabular}
    \vspace{1mm}
    \caption{Self-auxiliaries with single-task learning on \textbf{MNIST} dataset. Performances are measured in 95\% confidence intervals for classification accuracy. \label{tab6}}
  \end{center}
\end{table}

\end{document}